\documentclass[10pt,journal]{IEEEtran} % For two columns
\usepackage{standalone}

\usepackage{authblk}

\usepackage{color}

\usepackage{fullpage,graphicx,psfrag,amsmath,amsfonts,verbatim}
\usepackage[utf8]{inputenc}

\usepackage{flushend}

\usepackage{amsthm} % helps to define theorem-like structures
\usepackage{amssymb} % provides an extended symbol collection
\usepackage{wasysym} % \ocircle and \varhexagon
\usepackage{mathtools}
\usepackage{bm}

\usepackage{url}

\graphicspath{{./images/}}

\newcommand{\reals}{{\mbox{\bf R}}}

\newcommand{\Prob}{\mathop{\bf Prob}}
\newcommand*{\defeq}{\mathrel{\vcenter{\baselineskip0.5ex \lineskiplimit0pt
                     \hbox{\scriptsize.}\hbox{\scriptsize.}}}%
                     =}

\usepackage{bbm}               
\newtheorem{theorem}{Theorem}%[section]
\newtheorem{lemma}{Lemma}%[theorem]
\title{DJAM: distributed Jacobi asynchronous method for learning personal models}
\author[1,2]{Inês Almeida}
\author[1,2]{João Xavier, \textit{Member}, \textit{IEEE} 
	\thanks{This work was partially supported by Funda\c{c}\~{a}o para a Ci\^{e}ncia e Tecnologia (FCT), Portugal, under Project UID/EEA/50009/2013, and by grant PD/BD/135012/2017 from FCT. The authors would also like to express their gratitude to Dušan Jakovetić for his valuable input. Emails:
	{\tt\footnotesize almeida.inesb@gmail.com, jxavier@isr.ist.utl.pt}
	}
}
\affil[1]{Instituto Superior Técnico, Universidade de Lisboa, Lisbon, Portugal}
\affil[2]{Institute for Systems and Robotics (ISR), LARSyS, Lisbon, Portugal}

\begin{document}
%--------------------------------------%
\maketitle 

\begin{abstract}
%Abstract between 150 and 250 words: This has 174.
Processing data collected by a network of agents often boils down to solving an optimization problem. The distributed nature of these problems calls for methods that are, themselves, distributed. While most collaborative learning problems require agents to reach a common (or consensus) model, there are situations in which the consensus solution may not be optimal. For instance, agents may want to reach a compromise between agreeing with their neighbors and minimizing a personal loss function.
We present DJAM, a Jacobi-like distributed algorithm for learning personalized models. This method is implementation-friendly: it has no hyperparameters that need tuning, it is asynchronous, and its updates only require single-neighbor interactions. We prove that DJAM converges with probability one to the solution, provided that the personal loss functions are strongly convex and have Lipschitz gradient. We then give evidence that DJAM is on par with state-of-the-art methods: our method reaches a solution with error similar to the error of a carefully tuned ADMM in about the same number of single-neighbor interactions.
\end{abstract}
%--------------------------------------%
% !TEX root = main.tex

%--------------------------------------%
\section{Learning personal models} \label{sec:problem-statement}
Consider  $n$ agents, each with a personal loss function:  $f_i:\reals^p\rightarrow\reals$,  $\theta_i \mapsto f_i( \theta_i )$, for agent $i = 1, \ldots, n$. For example, $f_i(\theta_i)$ could be the loss of a model parameterized by $\theta_i$ on agent $i$'s personal dataset.  The agents are the nodes of an undirected, connected network.
% retirei o similar goals porque é explicado melhor já a seguir (e fazia confusão).

Each agent aims to find a model that minimizes both the mismatch with its neighbors' models and its personal loss.  More specifically, agents aim to solve
\begin{equation} \label{prob:model-prop}
\min_{\theta_1,\ldots,\theta_n}\hspace{.4cm}  
\frac{1}{2}\sum_{i<j}^n W_{ij}\|\theta_i-\theta_j\|^2
+\sum_{i=1}^n f_i(\theta_i),
\end{equation}
where $W = \left( W_{ij} \right) \in\reals^{n\times n}$ is a symmetric matrix that mirrors the topology of the network: $W_{ij}~\geq~0$ if agents $i$ and $j$ are connected in the network; $W_{ij}=0$ otherwise. 
The weight $W_{ij}$ controls the degree of agreement we want between agents $i$ and $j$: a larger $W_{ij}$ enforces more similarity between the corresponding agents' models.

Problem \eqref{prob:model-prop} can model a number of applications, including peer-network recomender systems, distributed (linear) classification~\cite{vanhaese2016personalized}, opinion propagation, and field estimation~\cite{eksin2012heuristic}; the latter is discussed in Section \ref{sec:exp}. The weights $W_{ij}$ may be chosen, for instance, based on the spatial distance between pairs of agents, or according to the similarity of their personal datasets.

\vspace*{2ex} 

\noindent\textbf{Closest related works.} Optimization problem~\eqref{prob:model-prop} has been addressed in~\cite{vanhaese2016personalized}. For convex loss functions $f_i$ that are quadratic, the authors suggest a distributed algorithm, which we refer to as Model Propagation Algorithm (MPA). In each round of MPA, an agent wakes up at random,  interacts with one of its neighbours, and both go back to sleep; the pattern repeats for the following rounds. MPA is an algorithm that is easy to implement because it is asynchronous (each agent has its own clock to wake up), has no parameter to tune,  and involves only single-neighbour interactions (the agent that wakes up does not need to coordinate message-passing with several neighbours). The authors in~\cite{vanhaese2016personalized} prove that MPA converges to the solution of~\eqref{prob:model-prop} in expectation (mean-value), for quadratic loss functions; for these functions, the iterations of the method we propose coincide with those of MPA.
For more general loss functions, those authors suggest a different algorithm, based on ADMM, which needs parameter tuning to reach optimal performance. This ADMM-based algorithm for collaborative learning (CL-ADMM), will be compared with our algorithm in Section~\ref{sec:exp}.

Problem~\eqref{prob:model-prop}, with the same kind of asynchronous single-neighbour interactions, can also be tackled by the algorithm proposed in~\cite{eksin2012heuristic}. In the language of~\cite{eksin2012heuristic}, this corresponds to having agents deviate from the ``rational" decision at each round (the rational decision would require each agent to interact will \textit{all} its neighbors). For such ``irrational" decisions, the authors show that, with probability one, the iterations of their algorithm will visit infinitely often a neighborhood of the solution of~\eqref{prob:model-prop}, although the iterations may continually escape that neighborhood.
Finally, a recent follow-up on \cite{vanhaese2016personalized} is~\cite{bellet2018personalized}, where a block coordinate descent method with broadcast communications is used to solve problem~\eqref{prob:model-prop}.

\vspace*{2ex}

\noindent\textbf{Contributions.} We show that a simple Jacobi-like distributed algorithm, which we call DJAM, can solve~\eqref{prob:model-prop} with the same kind of asynchronous single-neighbor interactions. DJAM, which can also be seen as a randomized block-coordinate method, has no parameters that need tuning. For continuously differentiable personal loss functions that are strongly convex and have Lipschitz gradient, that is, such that, for all $i$,
\begin{equation} (\nabla f_i(x) - \nabla f_i(y))^T (x-y) \geq m_i\|x-y\|^2 \label{assumpf1} \end{equation}
for some $m_i > 0$ and all $x,y$, and
\begin{equation} \|\nabla f_i(x)-\nabla f_i(y)\|\leq M_i\|x-y\| \label{assumpf2} \end{equation}
for some $M_i \geq 0$ and all $x, y$, we show that DJAM converges to the solution of~\eqref{prob:model-prop} with probability one. The values of $m_i$ and $M_i$ are used for proving convergence but need not be known when implementing DJAM.

DJAM improves on MPA not only because it applies to a larger class of functions than quadratics, but also because it converges in a stronger sense: as the proof of Theorem~\ref{theo:as-convergence} ahead shows, the DJAM iterations are uniformly bounded; thus, the convergence in expectation in~\cite{vanhaese2016personalized} follows by the dominated convergence theorem from our convergence with probability one. Our result only applies to a (somewhat) more restricted class of functions than the one of \cite{eksin2012heuristic}, but our convergence mode is stronger than the one of~\cite{eksin2012heuristic}.
Also, unlike in \cite{bellet2018personalized}, our method does not require knowing the values of $M_i$ upon implementation.

%Convergence with probability one is also stronger than the convergence mode in~\cite{eksin2012heuristic}, which guarantees only continual visits to a neighborhood of the solution, for the class of functions we consider.

\vspace*{2ex}

\noindent\textbf{Other related work.} Although~\cite{vanhaese2016personalized,eksin2012heuristic} are the closest works that we are aware of, many other distributed algorithms solve variations of problem~\eqref{prob:model-prop}. We now mention some representative work. 

A number of distributed algorithms allow agents to solve an underlying optimization problem by reaching \textit{consensus} on the solution. They use techniques ranging from distributed (sub)gradient descent~\cite{kvaternik2011lyapunov}, \cite{nedic2009dsg} to more elaborate techniques such as EXTRA~\cite{shi2015extra}, distributed ADMM~\cite{boyd2011admm,wei2013asyncadmm}, dual averaging~\cite{colin2016dualavg}, and distributed Augmented Lagrangean (AL) \cite{jakovetic2015linconv}.  Some algorithms aim at more specific optimization tasks such as distributed lasso regression \cite{mateos2010distlinreg}, distributed SVMs~\cite{forero2010distsvms}, and distributed RFVL networks \cite{scardapane2015rvfl}.
All of these methods aim at reaching consensus solutions---all agents converge to the \textit{same} value. Conversely, in problem~\eqref{prob:model-prop}, agents want to find \textit{different} (personalized) values.
%This group of distributed methods make all agents reach consensus on the solution---agents converge to the \textit{same} value. Instead, in problem~\eqref{prob:model-prop} agents want to find \textit{different} (personal) values.

The related problem of network lasso is dealt with in \cite{boyd2015networklasso}; however, the cost in \cite{boyd2015networklasso} puts a strong emphasis on neighbouring models being exactly equal, whereas in our case we want them to be similar, but not necessarily equal. The methods proposed in \cite{facchinei2015nonconvex} and \cite{necoara2016parallelcd} can tackle more general problems, but both require that agents communicate with all their neighbors before updating, while our method needs only communications between two agents at a time.
Problem \eqref{prob:model-prop} is also referred to as multitask problem; this problem is solved in \cite{sayed2014multitask} for a more restricted class of personal losses than ours.

%% {\color{red}We can write more stuff on this matter here if we need. No more than the space already occupied by this red text. This is just placeholder text. This is just placeholder text. This is just placeholder text.}

%--------------------------------------%
%--------------------------------------%
%--------------------------------------%
\section{DJAM} \label{sec:algorithm}
A naive Jacobi-like approach to solve~\eqref{prob:model-prop} would work as follows: at each round $t$, one agent $i$, picked at random, would update its model according to
\[
\theta_i(t+1) = \arg\min_{\theta_i} \frac{1}{2}\sum_{k\in\mathcal{N}_i} W_{ik}\|\theta_i-\theta_k(t)\|^2
+ f_i(\theta_i),
\]
where ${\mathcal N}_i$ is the set of neighbors of agent~$i$.
This naive approach, however, has a major drawback: it requires that agent $i$ communicates with \textit{all} its neighbors---to receive their up-to-date models $\theta_k(t)$---before updating its own model. Coordinating such message-passing, at each round, is cumbersome. A lighter scheme, involving only a single pair of agents at a time, is simpler to implement in practice, and requires fewer communications, at the expense of slowing down convergence. 

The key idea, which we borrow from~\cite{vanhaese2016personalized}, is to have each agent $i$ keep its own model $\Theta_i^i$ as well as (often outdated) versions of its neighbors' models, $\Theta_i^k$ for $k\in\mathcal{N}_i$. The versions of each pair of neighbors are updated whenever they communicate with each other. 
More specifically, at each round $t$, agent $i$ wakes up and chooses a neighbor $j\in\mathcal{N}_i$ to communicate with. They begin by exchanging information on their models, meaning that $\Theta_i^j(t+1)=\Theta_j^j(t)$ and $\Theta_j^i(t+1) = \Theta_i^i(t)$. All other variables remain unchanged. Afterwards, both agents update their own model via
\begin{equation} \label{eq:update}
\Theta_l^l(t+1) = \arg\min_{\theta_l} \frac{1}{2}\sum_{k\in\mathcal{N}_l} W_{lk} \|\theta_l - \Theta_l^k(t+1)\|^2 + f_l(\theta_l)
\end{equation}
for $l\in\{i,j\}$.

For the purpose of analyzing DJAM, we merge these two steps into a single one. Since the personal model $\Theta_i^i$ can be created at any time at agent $i$ via \eqref{eq:update}, it need not be stored. This means that, at round $t$ of DJAM, two neighboring agents $i$ and $j$ will \emph{compute and share} their own models with each other:
\begin{equation} \label{eq:iteration}
\Theta_i^j(t+1) = \arg\min_{\theta_j} \frac{1}{2}\sum_{k\in\mathcal{N}_j} W_{jk} \|\theta_j - \Theta_j^k(t)\|^2 + f_j(\theta_j),
\end{equation}
and similarly for $\Theta_j^i$. Mind that the right-hand side of \eqref{eq:iteration} is computed by agent $j$ and sent to agent $i$, who stores the result in the variable on the left-hand side of \eqref{eq:iteration}.\footnote{Updates \eqref{eq:update} and \eqref{eq:iteration} are equivalent apart from a minor technicality: The values of $\Theta_i^j$ are equal for all $t$, while those of the $\Theta_i^i$ may be (finitely) delayed from one implementation to the other. This detail does not affect the validity of our results.}

% !TEX root = main.tex

%--------------------------------------%
%--------------------------------------%
%--------------------------------------%
\section{Proof of convergence for DJAM} \label{sec:proof}
We now prove that DJAM, the algorithm with updates given by \eqref{eq:iteration}, converges with probability one to the solution of~\eqref{prob:model-prop}. We omit some laborious (but otherwise painless) technical steps that would make the notation and proofs too lengthy.

%${\mathcal E} = \left\{ (i, j)\,:\, i < j, i \sim j \right\}$
Let ${\mathcal E}$ be the set of edges of the network that links the agents. The network need not be fully connected: each agent is connected only to a subset of the remaining agents. We assume that at each round \textbf{(A1)} one edge of ${\mathcal E}$ is chosen at random, independently of previous choices; and \textbf{(A2)} each edge in $\mathcal{E}$ has a fixed, positive probability of being chosen. It is easy to verify that, under assumptions \textbf{(A1)} and \textbf{(A2),} each edge in~$\mathcal{E}$ is chosen infinitely often with probability one.

% We stack some supporting lemmas (with abridged proofs) before reaching the main result, Theorem~\ref{theo:as-convergence}. We begin with a simple property of the edge selection process.

%\begin{lemma}
%Under assumptions (A1) and (A2), each edge in~${\mathcal E}$ is chosen infinitely often.
%\label{theo:ij-io}
%\end{lemma}
%
%\begin{proof}
%This follows from the divergent part of the Borel-Cantelli Lemma \cite[Theorem~4.2.4]{chung2001course}.
%\end{proof}

%\begin{proof}
%We begin by showing that any given pair of linked agents will be selected i.o. with probability one. Given a pair $(i,j)\in\mathcal{E}$, define the sequence of independent random variables $Y_{(i,j)}(t)\defeq \indic_{\{(i,j)\}}(X(t))$. Since $p_{ij}>0$ by hypothesis,
%\begin{align*}
%\sum_{t=1}^\infty \Prob\left(Y_{(i,j)}(t) = 1\right) &= \sum_{t=1}^\infty p_{ij}\\
%&= \infty.
%\end{align*}
%It follows by the divergent part of the Borel-Cantelli lemma \cite{chung2001course} that
%\[
%\Prob\left(Y_{(i,j)}(t)=1\ \mbox{i.o.}\right)=1.
%\]
%But $Y_{(i,j)}(t)=1$ if and only if edge $(i,j)$ is selected at time $t$. It follows that edge $(i,j)$ will be selected infinitely often with probability one.
%
%Finally, we use the fact that a countable intersection of events with probability one has, itself, probability one \cite[Proposition~2.4~a)]{williams} to conclude all edges are selected i.o. with probability one.
%\end{proof}

The number of times a given edge $(i,j)$ is chosen between rounds $s$ and $t$ (with $s \leq t$) is a random variable defined as $S_{(i,j)}(s,t) := \sum_{\tau = s}^{t} Y_{(i,j)}(\tau)$,
where $Y_{(i,j)}(\tau)=1$ if edge $(i,j)$ is chosen at round $\tau$, and zero otherwise. 
We now define a useful family of stopping times $\left( T_m \right)_{m \geq 0}$. We let $T_0 := 0$ and, for $m \geq 0$,
\[
T_{m+1}\defeq \min \left\{ t \mid S_{(i,j)}(T_m + 1,t)\geq 1, \forall (i,j)\in\mathcal{E} \right\}.
\]
In words, $T_{m+1}$ is the first round after $T_m$  by which all edges have been chosen at least once. Assumptions \textbf{(A1)} and \textbf{(A2)} imply that any $T_m$ is finite for any $m$. Addtionally, $T_m\rightarrow\infty$ as $m\rightarrow\infty$ with probability one.

% Also, since $T_{m+1} \geq T_m + 1$ for any $m$, one finds that $T_m\rightarrow\infty$ as $m\rightarrow\infty$ with probability one.}

%Our next lemma states an intuitive property of $T_m$: the sequence $(T_m )_{m \geq 0}$ grows unbounded.
%
%\begin{lemma}
%Under assumptions (A1) and (A2), $T_m\rightarrow\infty$ as $m\rightarrow\infty$, with probability one.
%\label{theo:Tm-infinity}
%\end{lemma}
%
%\begin{proof}
%$T_{m+1} \geq T_m + 1$ for any $m$.
%\end{proof}
We first state an important consequence of assumptions~\eqref{assumpf1} and~\eqref{assumpf2} on each personal loss function $f_j$.

%\begin{proof}
%Choose any sequence of edges in $\mathcal{E}$. By construction, $T_{m+1} \geq T_m$, which means $T_m$ either converges or goes to infinity. Suppose it converges to $T<\infty$. Then, for all $\epsilon>0$, there exists some $m_0$ such that, for all $m\geq m_0$,
%\[
%\|T_m - T\| \leq \epsilon.
%\]
%By choosing $\epsilon<1$, we conclude all $T_m$ coincide on the value $T$ for all $m$ greater than some finite value. But, since only one pair of agents communicates at each time step, we must have $T_{m+1}\geq T_m + 1$ (to be precise, $T_{m+1} \geq T_m + \vert\mathcal{E}\vert$). This contradiction implies that $T_m$ diverges.
%\end{proof}

\begin{lemma} \label{theo:shrinking}
Let $w_j \defeq \sum_{k\in\mathcal{N}_j} W_{jk}$, and take the function
$F_j(x) = f_j(x) + \frac{1}{2}w_j\|x\|^2$. Note that $\nabla F_j$ is a bijective map (with inverse map $(\nabla F_j)^{-1}$) because, from standard convex theory, $\nabla f_j$ is. Then, for any $a$ and $b$,
\[
\|(\nabla F_j)^{-1}(a) - (\nabla F_j)^{-1}(b)\| \leq \left(m_j + w_j\right)^{-1}\|a - b\|.
\]
\end{lemma}

%\begin{lemma} \label{theo:shrinking}
%Consider the function $F_j(x) = f_j(x) + \frac{1}{2}\sum_{k\in\mathcal{N}_j} W_{jk}\|x\|^2$.
%Then, for any $a$ and $b$, 
%\[
%\|(\nabla F_j)^{-1}(a) - (\nabla F_j)^{-1}(b)\|
%\leq
%\left(m_j + w_j\right)^{-1}\|a - b\|
%\]
%where $w_j \defeq \sum_k W_{jk}$.
%\end{lemma}
%\begin{proof}
%This is proved using the fact that
%\[ (x-y)^\top (\nabla f_j(x) - \nabla f_j(y)) \geq m_j \|x-y\|^2 \]
%followed by the Cauchy-Schwartz inequality.
%\end{proof}

\begin{proof}
Choose $x = (\nabla F_j)^{-1}(a)$ and $y =(\nabla F_j)^{-1}(b)$. Clearly, $a - b = (\nabla f_j(x) - \nabla f_j(y)) + w_j (x-y)$. Multiplying both sides of this equality by $(x-y)^T$ yields $(x-y)^T (a - b) \geq (m_j + w_j) \|x-y\|^2$, where the inequality is due to the strong convexity of $f_j$, property~\eqref{assumpf1}. By the Cauchy-Schwartz inequality, $\|x-y\| \|a-b\|\geq(x-y)^\top (a-b)$,
and, thus, $\|x-y\|\|a-b\| \geq (m_j + w_j) \|x-y\|^2$. Inserting the definitions of $x$ and $y$  yields $\| a - b \| \geq (m_j + w_j ) \left\|  (\nabla F_j)^{-1}(a) -  (\nabla F_j)^{-1}(b) \right\|$.
%, the desired result.}
\end{proof}

We now give our main convergence result.

\begin{theorem}[DJAM converges with probability one]
Let $\Theta^*=(\Theta_1^*,\ldots,\Theta_n^*)$ be the  solution of \eqref{prob:model-prop}. Let $\Theta_i^j(t+1)$, $i=1,\ldots,n$, $j\in\mathcal{N}_i$, be updated via \eqref{eq:iteration} whenever edge $(i,j)$ is chosen at round $t$, and similarly for $\Theta_j^i(t+1)$. 
Then,  for any pair of agents $(i,j)$, $\Theta_i^j(t)\rightarrow \Theta_j^*$ as $t \rightarrow \infty$, with probability one.
\label{theo:as-convergence}
\end{theorem}

\begin{proof}
Let $F_j$, and $w_j$ be defined as in Lemma \ref{theo:shrinking}. Suppose edge $(i,j)$ is chosen at time $t$.  %the update rule \eqref{eq:iteration} can be rewritten in terms of $F_j$ as
It can be verified from \eqref{eq:iteration} and from the first order condition for optimality that
$\Theta_i^j(t~+~1)~=~(\nabla F_j)^{-1} \left(\sum_{k\in\mathcal{N}_j}W_{jk}\Theta_j^k(t)\right)$, where $F_j$ is as defined in Lemma~\ref{theo:shrinking};
similarly, we have  that $\Theta_j^* = (\nabla F_j)^{-1} \left(\sum_{k\in\mathcal{N}_j}W_{jk}\Theta_k^*\right)$ for each component of the solution.

Lemma \ref{theo:shrinking} allows us to find that
\begin{eqnarray}
\lefteqn{\|\Theta_i^j(t+1) - \Theta_j^*\|} \nonumber \\
& \leq & (m_j + w_j)^{-1} \|\sum W_{jk} (\Theta_j^k(t) - \Theta_k^*)\| \nonumber \\
& \leq & (m_j + w_j)^{-1} \sum W_{jk} \|\Theta_j^k(t) - \Theta_k^* \| \nonumber \\
& \leq & (m_j + w_j)^{-1} w_j \max_k \|\Theta_j^k(t) - \Theta_k^* \| \nonumber \\
& \leq & (m_j + w_j)^{-1} w_j V(t), \label{ineq:step}
\end{eqnarray}
where $V(t) := \max_{l,k} \| \Theta_l^k(t) - \Theta_k^* \|$ is the maximum error at round $t$ between the agents' estimates and the solution.

%{\color{blue}$V(t)$ is the maximum error at round $t$ between the agents' estimates and the solution,}
%\begin{equation}
%\color{blue} V(t) := \max_{l,k} \| \Theta_l^k(t) - \Theta_k^* \|. \label{defV}
%\end{equation}

If edge $(i,j)$ is chosen at round $t$, we have, by the derivation above, that \begin{equation} \|\Theta_i^j(t+1) - \Theta_j^*\| \leq  V(t). \label{keyineq} \end{equation} If that edge is not chosen, then $\Theta_i^j(t+1) = \Theta_i^j(t)$ and, by definition of $V(t)$, $\|\Theta_i^j(t+1) - \Theta_j^*\| \leq V(t)$. We conclude that~\eqref{keyineq} holds for any pair $(i,j)$ and, so, $V(t+1)~\leq~V(t)$. Since $V(t) \geq 0$, the limit (which is a random variable) $V \defeq \lim_{t\rightarrow\infty} V(t)$
is thus always well defined. The goal of the proof is to show that $V = 0$ with probability one.

%Consider now a sequence of randomly-chosen edges in which all edges of the graph appear infinitely often. Note that, by Lemma~ \ref{theo:ij-io}, such sequences exist with probability one; this means that the following reasoning holds with probability one. Since, by definition of $T_{m+1}$, all edges $(i,j)$ were selected at least once between $T_m$ and $T_{m+1}$, we know that

Recall that $T_{m+1}$ denotes the first round after $T_m$ by which all edges were selected at least once. It should be clear to the reader that the remainder of the proof holds almost surely, since $T_m$ is finite for all $m$ with probability one. Suppose edge $(i,j)$ was selected at round $T_m + s$. Then
\begin{equation}
\|\Theta_i^j(T_m+s) - \Theta_j^*\| \leq (m_j + w_j)^{-1} w_j V(T_m),
\label{eq:Tm_plus_s}
\end{equation}
cf. inequality \eqref{ineq:step}. Since, by definition of $T_m$ and $T_{m+1}$, all edges $(i,j)$ in the graph were selected at least once between $T_m$ and $T_{m+1}$, inequality \eqref{eq:Tm_plus_s} holds for all $(i,j)\in\mathcal{E}$ when $T_m + s = T_{m+1}$. In other words,
\[
\|\Theta_i^j(T_{m+1}) - \Theta_j^*\| \leq  (m_j + w_j)^{-1} w_j  V(T_m)
\]
for all edges $(i,j)\in\mathcal{E}$. It follows, by the definition of $V(T_{m+1})$, that
\begin{equation} V(T_{m+1})\leq \beta V(T_m) \label{ineqV}, \end{equation} 
where $\beta~\defeq~\max_i\ \{ (m_i + w_i)^{-1} w_i\} \in [0,1)$.

Let us take the limit $m\rightarrow\infty$ in~\eqref{ineqV}. We know that $T_m\rightarrow\infty$, and it follows from~\eqref{ineqV} that $V\leq \beta V$. Thus, owing to  $0 \leq \beta < 1$, we must have $V=0$, which implies $\Theta_i^j(t)\rightarrow \Theta_j^*$ as $t \rightarrow \infty$, for any $(i,j)$.
\end{proof}

\noindent{Inequality~\eqref{ineqV} implies that $V(t) \leq V(0)$. The iterations $\left( \Theta_k^l(t) \right)_{t \geq 0}$ are, thus, uniformly bounded. We conclude (by the dominated convergence theorem) that our result implies the convergence in expectation result in~\cite{vanhaese2016personalized}.}

%\noindent\textbf{Convergence in expectation.} Note that~\eqref{ineqV} implies $V(t) \leq V(0)$ for all $t$. This means that the iterations $\left( \Theta_k^l(t) \right)_{t \geq 0}$ are uniformly bounded. Since, $\Theta_k^l(t) \rightarrow \Theta_k^*$ with probability one  as $t \rightarrow \infty$, we conclude (by the dominated convergence theorem) that the convergence also takes place in expectation. We thus obtain the convergence mode of~\cite{vanhaese2016personalized}, as a particular case.

%Our result is strictly more general than \cite[Theorem~1]{vanhaese2016personalized}. To see this, notice that the quadratic form $\|\theta_i - \bar{\theta}_i \|^2$ is strongly convex; we thus have a.s. convergence in the quadratic-loss scenario. Finally, the dominated convergence theorem implies convergence in expected value.
% !TEX root = main.tex
%--------------------------------------%
%--------------------------------------%
%--------------------------------------%
\section{Field estimation example} \label{sec:exp}
%\vspace*{2ex}
\noindent\textbf{Setup.} Following~\cite{eksin2012heuristic}, we consider a field estimation setup that leads to a problem of the form~\eqref{prob:model-prop}. The $n$ agents are spread in a region and wish to profile a certain quantity, say, temperature, over the region: agent~$i$ cares only about the value of the quantity at its location, $\theta_i$. Assume that the true values of the temperatures, $\theta = (\theta_1,\ldots, \theta_n)$, are drawn from a prior distribution: a normal distribution with known mean and covariance $\Sigma$; as in~\cite{eksin2012heuristic}, we \emph{assume} that the off-diagonal elements of $\Sigma^{-1}$ match the sparsity of the network, that is, $(\Sigma^{-1})_{ij}>0$ if and only if $(i,j)\in\mathcal{E}$. %{\color{blue}This means that temperatures of neighbouring agents are similar to one another and, therefore, that the underlying temperature profile is smooth across the network.}
Agent~$i$ measures $y_i = \theta_i + \nu_i$, where $\nu_i$ models identically distributed sensor noise (for simplicity), which is independent across agents. 

\vspace*{2ex}

\noindent\textbf{MAP estimation.} A maximum 
a posteriori (MAP) approach seeks the $\theta = ( \theta_1, \ldots, \theta_n )$ that maximizes $
\sum_{i=1}^n \log\Prob(y_i \mid \theta_i) + \log\Prob(\theta_1,\ldots,\theta_n)$; or, equivalently, the $\theta$ that minimizes
\begin{equation}
\frac{1}{2}\left(
\sum_{i\sim j} \sigma_{ij}(\theta_i - \theta_j)^2 +
\sum_{i=1}^n \sigma_{ii}\theta_i^2
\right)
+
\sum_{i=1}^n \phi(y_i-\theta_i), \label{lastprob}
\end{equation}
where $\sigma_{ij} := (\Sigma^{-1})_{ij}$ and $\phi$ depends on the distribution of the noise $\nu_i$. We let $\phi$ be a Huber penalty function to handle outliers~\cite{huber1981huber}.
Finally, defining the personal loss functions as $f_i(\theta_i) \defeq \phi(y_i-\theta_i) + \frac{1}{2} \sigma_{ii}\theta_i^2$ puts~\eqref{lastprob} in the form~\eqref{prob:model-prop}. Also, assumptions~\eqref{assumpf1} and~\eqref{assumpf2} hold.

\vspace*{2ex}

\noindent\textbf{Results: comparing DJAM with CL-ADMM.}
Since the algorithm MPA from~\cite{vanhaese2016personalized} applies only to quadratic functions, we use the ADMM-based algorithm CL-ADMM from~\cite{vanhaese2016personalized} to compare with DJAM. Note that both CL-ADMM and DJAM converge to the solution with probability one. The algorithm CL-ADMM, however, being based on ADMM, has a parameter to tune---the parameter in the quadratic penalization part of the augmented Lagrangian function. This parameter, which we refer to as $\rho$, is known to affect noticeably the convergence speed of ADMM.

\begin{figure}[t] % changing this to [t] buys us a lot of room!!
\centering \includegraphics[width=0.5\textwidth]{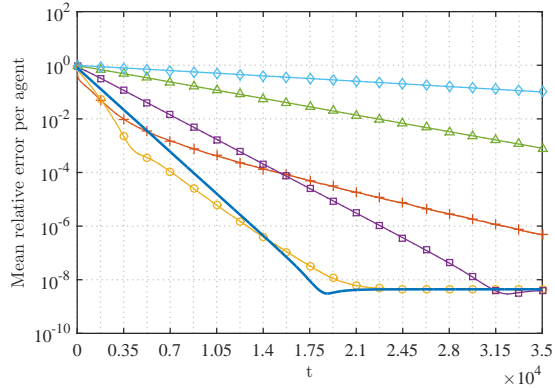}
\caption{Mean relative error $\left( \left\| \Theta_i^i(t) - \Theta_i^* \right\| / \left\| \Theta_i^* \right\| \right)$ per agent on the smooth field estimation problem for an instance with $n = 30$ agents; the mean was obtained by averaging over 100 Monte Carlo trials. Our method, DJAM, corresponds to the blue line with no markers. The other lines correspond to CL-ADMM with $\rho$ equal to 0.1~(red~$+$), 0.316~(yellow~$\ocircle$), 1.0~(violet~$\square$), 3.16~(green~$\vartriangle$), and 10~(cyan~$\lozenge$). All methods stop improving after reaching a relative error slightly above $10^{-9}$, which we believe is due to rounding errors.}
\label{fig:huber}
\end{figure}

The results for a field estimation instance are shown in Figure~\ref{fig:huber}. It shows, across rounds~$t$,  the relative error between an agent's private model $\Theta_i^i(t) $ and the solution component $\Theta_i^*$: $\left\| \Theta_i^i(t) - \Theta_i^* \right\| / \left\| \Theta_i^* \right\|$. The relative error was averaged over agents and over 100 Monte Carlo trials where, in each Monte Carlo run, we choose a different set of edges along time.

Figure~\ref{fig:huber} confirms that the speed of convergence of CL-ADMM varies with the parameter $\rho$ noticeably. 
In fact, we verified in other simulations (omitted due to lack of space) that the optimal $\rho$ varied significantly with the number of agents, with the range of values for $\sigma_{ij}$, and with the noise distribution---we found the optimal~$\rho$ for those simulations by careful hand-tuning. In contrast, DJAM is on par with the best $\rho$ in Figure~\ref{fig:huber}, and needs no parameter tuning. 

%--------------------------------------%
%--------------------------------------%
%--------------------------------------%
%\section{Conclusions} \label{sec:final}
%Typical distributed optimization problems require agents to reach a common (or consensus) solution. However, there are certain applications in which it is preferable that agents learn personalized solutions. We presented DJAM, a distributed Jacobi assynchronous (single-neighbour) method for tackling such problems. We proved that this method converges with probability one to the centralized solution when the personal loss functions are strongly convex and have Lipschitz gradient. We then compared our method with a distributed, asynchronous version of ADMM, and verified that our method is comparable with ADMM in terms of performance, with the advantage of being parameter-free, while the performance of ADMM strongly depends on its penalty parameter.

\pagebreak
\bibliographystyle{ieeetr}
\bibliography{./references/model_prop}
\end{document}